\documentclass{article}
\usepackage{graphicx}
\usepackage{hyperref} 
\usepackage[T1]{fontenc}

\usepackage{amsmath, amsthm, amssymb}
\usepackage{tikz}
\usepackage{subcaption}
\usetikzlibrary{shapes.geometric}
\usetikzlibrary{angles,quotes}
\usetikzlibrary{calc} 
\usetikzlibrary{arrows.meta, positioning}
\usepackage{cleveref}

\usepackage{tcolorbox}
\tcbuselibrary{skins}
\usepackage{dashrule} 
\usepackage{xcolor}
\usepackage{wrapfig}
\usepackage{fullpage}
\usepackage{natbib}

\usepackage{booktabs}
\usepackage{nicematrix}

\newtheorem{theorem}{Theorem}

\newtheorem{lemma}[theorem]{Lemma}

\newtheorem{remark}[theorem]{Remark}
\newtheorem{claim}[theorem]{Claim}

\title{Transformer Heads Looking for Order}

\author{%
Jasper van Doornmalen\\ IMC UC\\\texttt{mvandoornmalen@uc.cl}
\and 
Alexander Kozachinskiy\\ CENIA\\ \texttt{alexander.kozachinskyi@cenia.cl} 
\and 
Corinna Mathwieser \\ IMC UC \\ \texttt{corinna.mathwieser1@uc.cl}
\and 
Tomasz Steifer\\Centre for Credible AI\\\texttt{tsteifer@ippt.pan.pl} 
\and 
Felipe Urrutia \\ IMC UC \\\texttt{felipe.urrutia@uc.cl}
\and 
José Verschae  \\ IMC UC \\ \texttt{jverschae@uc.cl}
\and 
Przemys{\l}aw Andrzej Wa{\l}\c{e}ga \\ Queen Mary University of London \\\texttt{p.walega@qmul.ac.uk}%
}

\begin{document}
\maketitle
\begin{abstract}%
In this note, we show that the problem of checking, whether a sequence of bits is ordered, is not doable by 1-head 1-layer transformers but is doable by a 2-head 1-layer transformer.  Unlike  similar previous results, our results assume  the model where transformers have an output MLP.

\end{abstract}

\section{Introduction}

Recently,~\cite{tesfaye2026two} gave an example of a task which is not solvable by 1-layer 1-head transformers but is solvable by a 1-layer 2-head transformer. Later, \cite{rajaraman2026head} extended this result by showing that $k$-bit parity is computable by a 1-layer $k$-head transformer while it is not computable by 1-layer $(k-1)$-head transformers.

These results are for \emph{attention-only} transformers -- those that do not have an output MLP. With the output MLP, these examples no longer work -- in particular, because they have a fixed input length when the output MLP can simply memorize the function.

We partially extend these results to 1-layer transformers with the output MLPs.
We give an example of a language (or, equivalently, a sequence of  Boolean functions)  which is computable with a 1-layer 2-head transformer with an output MLP, but is not computable by 1-layer 1-head transformers with output MLPs. This is the language consists of ordered binary strings -- that is,  strings where there is no 1 before 0.

\medskip

In the construction of a transformer with 2 heads, one head essentially looks for the position of the last 0, and the second head looks for the position of the first 1. The output MLP checks that they go one after another.

The lower bound first uses a technique of~\cite{kozachinskiy2026parity} to ``embed'' 1-head 1-layer transformers into the first-order theory of reals with addition and order. While~\cite{kozachinskiy2026parity} show a similar result for any number of heads with the use of multiplications, we observe that for 1 head multiplications can be avoided.

The main technical component of the lower bound is a geometric lemma, lower-bounding the number of halfspaces, needed to separate ordered points of a Boolean cube from the non-ordered ones. Similar geometric problems has been previously considered in the area of \emph{relaxation complexity}. However, the difference there is that ordered points of a Boolean cube have to be separated from all the other integer points, not only points of a Boolean cube. Then the problem for ordered points of a Boolean cube is equivalent to the problem for the vertices of a simplex where the exact asymptotic is known~\cite{averkov2026relaxation}. However, known results from relaxation complexity are not sufficient for our applications.

Very recently, \cite{hheads} have established, for every $k\ge 2$, that there exists a problem, doable with 1-layer $k$-head transformer, and not doable with $1$-layer $(k-1)$-head transformer (in the model with output MLP). This problem is checking whether the number of 1s in the first half of the input word is at least the $k$-th power of the number of 1s in the second half of the word.

\medskip

 On a more technical level, we assume  the standard softmax transformer model including the residual output MLP network, with the dimension and the parameters independent of the input length, while the input embedding is arbitrary and can be length-dependent. In fact, in our lower bound we do not even assume that the input embedding is additive, that is, decomposable into the sum of a token embedding and a positional encoding.
 
 At the same time, the input embedding in our construction with 2 heads is additive, with a simple length-independent positional encoding, consisting of the absolute value of the position and of its square. In comparison, the upper bound construction of~\cite{hheads} uses an length-independent but a non-additive input embedding. It thus seems open to obtain a separation between 1-layer $(k-1)$-head transformers and 1-layer $k$-head transformers for $k\ge 3$ for the additive input embeddings (preferably, the lower bound should hold even for non-additive embeddings while the upper bound should be attained via an additive one).

\section{Preliminaries}
\label{sec:transformers}

\paragraph{Basic Notation.} 
For a vector $x \in \mathbb{R}^d$ we write $x_i$ for its $i$-th
coordinate, and $(\mathbb{R}^d)^*$ for the set of all finite sequences of vectors in
$\mathbb{R}^d$. 
We study a family of functions 
$\mathrm{ORDERED} = \{\mathrm{ORDERED}_n\}_{n\geq 1}$,
where
\[\mathrm{ORDERED}_n:\{0, 1\}^n \to \{0, 1\}, \qquad \mathrm{ORDERED}_n(x_1, \ldots, x_n) = \begin{cases}1 & x_1 \le x_2 \le \ldots x_n,\\ 0& \text{otherwise}.\end{cases} \]

\paragraph{Attention Layers.}
A $d$-dimensional, $H$-head \emph{attention layer} is a
function $L \colon (\mathbb{R}^d)^* \to (\mathbb{R}^d)^*$ parameterised by query,
key, and value matrices $Q^{(k)}, K^{(k)}, V^{(k)} \in \mathbb{R}^{d \times d}$,
with $k = 1, \ldots, H$, a heads-mixing matrix  $W_O \in \mathbb{R}^{d \times dH}$,
and parameters $W_1, W_2 \in \mathbb{R}^{d \times d}$, $b_1, b_2 \in \mathbb{R}^d$ of the output MLP.
On input $(\alpha_1, \ldots, \alpha_n) \in (\mathbb{R}^d)^n$, the layer
computes for each head $k =1 ,\dots, H$: the \emph{attention logits}  of the $j$th to $i$th position $L^{(k)}_{ij} \in \mathbb{R}$  and \emph{head values}  in $j$th position $h_j^{(k)} \in \mathbb{R}^d$ (for $i, j = 1, \ldots, n$) as follows:

\medskip
\noindent\begin{minipage}{0.38\textwidth}
\begin{equation}
\label{eq_logits}
L^{(k)}_{ij} = \frac{\langle K^{(k)} \alpha_i, Q^{(k)} \alpha_j \rangle}{\sqrt{d}},
\end{equation}
\end{minipage}
\hfill
\begin{minipage}{0.45\textwidth}
\begin{equation}
\label{eq_headvalues}
h_j^{(k)} = \frac{\sum_{i = 1}^{n } \exp\{L^{(k)}_{ij}\}\, V^{(k)} \alpha_i}
                  {\sum_{i = 1}^{m } \exp\{L^{(k)}_{ij}\}}.
\end{equation}
\end{minipage}
\medskip

The head values are
combined and passed through a position-wise feed-forward network:

\medskip
\noindent\begin{minipage}{0.32\textwidth}
\begin{equation}
\label{eq_multihead}
h_j = W_O \begin{pmatrix} h_j^{(1)} \\ \vdots \\ h_j^{(H)} \end{pmatrix},
\end{equation}
\end{minipage}
\hfill
\begin{minipage}{0.65\textwidth}
\begin{equation}
\label{eq_ffn}
\beta_j = W_2 \cdot \mathrm{ReLU}\bigl( W_1(h_j + \alpha_j) + b_1 \bigr) + b_2,
\end{equation}
\end{minipage}
\medskip

\noindent  where $\mathrm{ReLU}(x_1, \ldots, x_d) = (\max\{0, x_1\}, \ldots, \max\{0, x_d\})$.
The output sequence is $(\beta_1, \ldots, \beta_n) = L(\alpha_1, \ldots, \alpha_n)$.

\paragraph{Transformers.}
We consider transformers that map sequences of tokens into the next, most probable token. We restrict our attention to 1-layer transformers.
A 1-layer, $H$-head, $d$-dimensional \emph{transformer} over a finite
vocabulary $\mathcal{V}$ (a set of \emph{tokens} which includes  symbol $\bot$) is a function
$T \colon \mathcal{V}^* \to \mathcal{V}$ specified by an
$H$-head $d$-dimensional attention layer
$L$, an input embedding
$\mathrm{E} \colon \mathcal{V} \times \mathbb{N}^2 \to \mathbb{R}^d$, and an
output distribution matrix $W \in \mathbb{R}^{|\mathcal{V}| \times d}$. On input
$x_1 \dots x_n \in \mathcal{V}^n$, it  applies position-wise the input embedding
$\alpha_i = \mathrm{E}(x_i, i, n)$, then applies the attention layer computing
$(\beta_1, \ldots, \beta_n) =L(\alpha_1, \ldots, \alpha_n)$,
and outputs
\begin{equation*}
\label{eq:output_token}
T(x_1, \ldots, x_n) = \arg\max_{x \in \mathcal{V}}\,
\bigl(\mathrm{softmax}(W \beta_n)\bigr)_x,
\end{equation*}
with the convention that $T(x_1, \ldots, x_n) = \bot$ if the argmax is not unique.


We say that a transformer $T$ \emph{computes} a sequence $\{f_n\}_{n \in \mathbb{N}}$ of
Boolean functions $f_n \colon \{0, 1\}^n \to \{0, 1\}$ if
$\{0, 1\} \subseteq \mathcal{V}$ and $T(x) = f_n(x)$, for every $n \in \mathbb{N}$
and every $x \in \{0, 1\}^n$.

\section{Main Result}

\begin{theorem} $\mathrm{ORDERED}$ can be computed by a 2-head 1-layer transformer, but it cannot be computed by a 1-head 1-layer transformer.
\end{theorem}
\begin{proof}

A construction of a 2-head 1-layer transformer for $\mathrm{ORDERED}$ is given in Subsection \ref{subsec_2head}. We know proceed to the lower bound. We require the following lemma whose analogue for 1-layer transformers with any number of heads was obtained in~\cite{kozachinskiy2026parity}. However, in general case, the formula $\Phi$ requires multiplication. We notice in the 1-head case, multiplications can be avoided.
\begin{lemma} 
\label{lemma_ford_lineal}
Assume there is a 1-head 1-layer transformer $T$ of dimension $d$, computing  a sequence of Boolean functions $\{f_n\}_{n = 1}^\infty$. Then
    there exists a first-order formula $\Phi(\mathbf{z}_1, \ldots, \mathbf{z}_r)$ in the interpretation $(\mathbb{R}, <,+)$  such that for all $n$ there exist $r$  polynomials $l_1, \ldots, l_r \in\mathbb{R}[x_1,\ldots, x_{n}]$ of degree at most 1
    such that for any $x\in\{0, 1\}^{n}$, we have $f_n(x) = 1$ if and only if $\Phi(l_1(x), \ldots, l_r(x)) = 1$.
\end{lemma}
\begin{proof}
We construct two formulas $\Phi_0, \Phi_1$ that work for inputs where the last bit is fixed to 0 and 1, respectively. The final formula $\Phi$ can be obtained by introducing a fresh variable  $\mathbf{z}$ which will be substituted by $x_n$, and writing
\[\Phi := ((\mathbf{z}  = 0) \to \Phi_0)\land ((\mathbf{z} = 1) \to \Phi_1).\]

It remains now to construct $\Phi_0$ (the formula $\Phi_1$ can be constructed similarly). The output of the transformer is determined by the maximal coordinate of the vector $W\beta_n$, where $W$ is the  output-distribution matrix, and $\beta_n$ is the output vector of the layer in the $n$-th position. More specifically, in order for the output of the transformer to be 1, the value of the $1$-coordinate of $W\beta_n$ has to be strictly larger than the corresponding coordinate for all the other tokens of the vocabulary:
\begin{equation}
\label{eq_ineqs}
    (W\beta_n)_1 > (W\beta_n)_\sigma, \qquad \sigma\in \mathcal{V}\setminus \{1\}.
\end{equation}
In turn, $\beta_n$ is computed as in \eqref{eq_headvalues} for $j = n$:
\begin{equation}
    \label{eq_betan}
    \beta_n = W_2 \cdot \mathrm{ReLU}\bigl( W_1(h_n + \alpha_n) + b_1 \bigr) + b_2.
\end{equation}
Firstly, we introduce $d$ variables $\mathbf{u}_1, \ldots, \mathbf{u}_d$, corresponding to coordinates of the vector $h_n + \alpha_n$. We write a $(\mathbb{R}, <, +)$-formula $\Psi(\mathbf{u}_1, \ldots, \mathbf{u}_d)$, expressing a fact that the vector $\beta_n$, computed as in \eqref{eq_betan} with 
\[\begin{pmatrix}
    \mathbf{u}_1 \\ 
    \mathbf{u_2} \\
    \vdots \\
    \mathbf{u_d}
\end{pmatrix}\]
in place of $h_n + \alpha_n$, satisfies inequalities \eqref{eq_ineqs}. This is doable because all the vectors in question have fixed dimension, matrices and vectors $W, W_1, W_2, b_1, b_2$ consists of finitely many fixed real numbers, and the ReLU operation is expressible in $(\mathbb{R}, <, +)$:
\[y = \mathrm{ReLU}(x) \iff (x\ge 0 \to y = x) \land (x < 0 \implies y = 0).\]
By a result of~\cite{ferrante1975decision}, the interpretation $(\mathbb{R}, <, +)$ admits a quantifier elimination, and thus the formula $\Psi$ can be assumed to be quantifier-free. For any input length $n$, for any transformer input $x \in\{0, 1\}^n$, if we compute the values of coordinates of $h_n + \alpha_n$ on this input and put these values into the formula $\Psi$, it will output the value  $f_n(x)$.

To finish the argument,
it is enough to show that for any input length $n$, for inputs with the last bit $x_n$ fixed to $0$, there exist $d+1$ polynomials $\ell_0, \ell_1,\ldots, \ell_d\in\mathbb{R}[x_1, \ldots, x_n]$ of degree at most 1 such that:
\begin{equation}
    \label{eq_hn}
    h_n + \alpha_n =\begin{pmatrix}
    \ell_1(x)/\ell_0(x) \\
    \ell_2(x) / \ell_0(x) \\
\vdots \\ 
    \ell_d(x) /\ell_0(x)
\end{pmatrix},
\end{equation}
where $\ell_0(x) > 0$ for all $x\in\{0, 1\}^n$. Indeed, then we can introduce $d + 1$ variables $\mathbf{v}_0, \mathbf{v}_1, \ldots, \mathbf{v}_d$ and consider a formula:
\[\Phi_0(\mathbf{v}_0, \ldots, \mathbf{v}_d) = \Psi(\frac{\mathbf{v}_1}{\mathbf{v}_0}, \ldots, \frac{\mathbf{v}_d}{\mathbf{v}_0}).\]
As written, the right-hand side formula will be a Boolean combination of linear inequalities with fractions $\frac{\mathbf{v}_1}{\mathbf{v}_0}, \ldots, \frac{\mathbf{v}_d}{\mathbf{v}_0}$. We can simply multiply them by their common\footnote{It is critical here that there is just 1 attention head; if there were already 2, we would have had two different denominators, and after multiplication, we would no longer have linear inequalities.} denominator $\mathbf{v}_0$ (taking into account that the function $\ell_0(x)$ that will be substituted in place of $\mathbf{v}_0$ takes only positive values on the transformer inputs) to obtain equivalent linear inequalities in $\mathbf{v}_0, \ldots, \mathbf{v}_d$. It remains to substitute $\ell_0, \ell_1,\ldots, \ell_d$ in place of $\mathbf{v}_0, \mathbf{v}_1\ldots, \mathbf{v}_d$.

It remains to show that the representation as in \eqref{eq_hn} takes place.
It is enough to obtain such a representation for the vector $h_n^{(1)}$ from \eqref{eq_headvalues} because $h_n + \alpha_n = W_O h_n^{(1)} + \alpha_n$, the matrix $W_O$ is a constant $d\times d$ matrix, and $\alpha_n = E(0, n, n)$ does not depend on $x\in\{0, 1\}^n$.

Now, the vector $h_n^{(1)}$ is defined as:
\begin{equation}
\label{eq_hn1}
h_n^{(1)} =
\frac{\sum_{i = 1}^{n } \exp\{L^{(1)}_{in}\}\, V^{(1)} \alpha_i}
                  {\sum_{i = 1}^{n } \exp\{L^{(1)}_{in}\}}.    
\end{equation}
By definition, expressions
\[ \exp\{L^{(1)}_{in}\}\, V^{(1)} \alpha_i, \qquad  \exp\{L^{(1)}_{in}\}\]
are functions of $i, n, x_i, x_n$. Under the fixation $x_n = 0$, they become functions of just $i,n, x_i$, and thus can be written as:
\[ \exp\{L^{(1)}_{in}\}\, V^{(1)} \alpha_i = x_i \rho^{1}_{in} + (1 - x_i)\rho^{0}_{in}, \qquad \exp\{L^{(1)}_{in}\} = x_i \theta^1_{in} + (1 - x_i) \theta^0_{in}\]
for some $\rho^1_{in}, \rho^0_{in}\in\mathbb{R}^d, \theta^1_{in}, \theta^0_{in}\in\mathbb{R}$. We see that both the numerator and the denominator in \eqref{eq_hn1} depend linearly on $x_1, \ldots, x_n$, and the denominator is a sum of positive numbers for every $x\in\{0,1 \}^n$, as required.

\end{proof}

That goal now is to show that such formula $\Phi$ as in Lemma \ref{lemma_ford_lineal} cannot exist for $\mathrm{ORDERED}$. We do this through the following geometric lemma. We say that an $n$-dimensional vector $(x_1, \ldots, x_n)\in\mathbb{R}^n$ is binary if $x_i\in\{0, 1\}$ for all $i = 1, \ldots, n$, and is ordered if:
\[x_1 \le x_2 \le \ldots \le x_n.\]
\begin{lemma}
    \label{lem_geom} For $n\in\mathbb{N}$, let $\theta_n$ denote the minimal $k\in\mathbb{N}$ such that there exist $k$ halfspaces $H_1, \ldots, H_k\in\mathbb{R}^n$, satisfying the following properties:
    \begin{itemize}
        \item all ordered $n$-dimensional binary vectors belong to $H_1 \cap H_2 \cap \ldots \cap H_k$;
        \item none of the non-ordered $n$-dimensional binary vectors belongs to $H_1 \cap H_2 \ldots \cap H_k$;
    \end{itemize}
    (each of the halfspaces can be with or without the border).

    It holds that:
    \[\lim\limits_{n\to\infty} \theta_n = +\infty.\]
\end{lemma}
The proof of Lemma \ref{lem_geom} is deferred to Section \ref{subsec_geom}. Now, assume for contradiction that $\mathrm{ORDERED}$ is computable by a 1-head 1-layer transformer. Then a formula $\Phi(\mathbf{z}_1, \ldots, \mathbf{z}_r)$ as in Lemma \ref{lemma_ford_lineal} exists for $\mathrm{ORDERED}$.  By the result of~\cite{ferrante1975decision}, we may assume that $\Phi(\mathbf{z}_1, \ldots, \mathbf{z}_r)$ is quantifier-free. Let $C$ be the number of its atomic subformulas. Without loss of generality, we may assume that all of them are non-strict inequalities. We will show that $\theta_n \le C$ for infinitely many $n$, contradicting Lemma \ref{lem_geom}.

 Take any $n$. There are $\ell_1, \ldots, \ell_r\in \mathbb{R}[x_1,\ldots, x_n]$ of degree at most 1 such that
  \begin{equation}
      \label{eq_ord}
      \mathrm{ORDERED}_n(x) = \Phi(\ell_1(x), \ldots, \ell_r(x))
  \end{equation}
for all $x\in\{0, 1\}^n$.  Each atomic subformula in \eqref{eq_ord} is a non-strict linear inequality in $x_1, \ldots, x_n$. Each binary vector $x\in\{0, 1\}^n$ gives rise to a sequence of $C$ bits which we will call the $\Phi$-pattern of $x$, encoding which atomic subformulas in \eqref{eq_ord} are true on $x$. The $\Phi$-pattern of a binary vector $x$ determines the value of the formula $\Phi$ on it, and hence an ordered binary vector and a non-ordered binary vector cannot have the same $\Phi$-pattern. Out of $n +1$ ordered  binary vectors, there exist at least $m \ge (n+1)/2^C$ with the same  $\Phi$-pattern. Let these $m$ vectors be $x^1, \ldots, x^m$. Their $\Phi$-pattern defines an intersection of $C$ halfspaces $H_1, \ldots, H_C$ (some might be with border and the other without) that contains $x^1, \ldots, x^m$  but does not contain any non-ordered binary vector.

Vectors $x^1, \ldots, x^m$  are all equal  to 0 in some initial positions and to 1 in some final positions. The rest of positions can be cut into $m - 1$ blocks $B_1, B_2, \ldots, B_{m - 1}$ so that: 
\begin{center}
  \begin{tabular}{c|c|c|c|c|c}
     &  $B_1$ & $B_2$ & $\ldots$ & $B_{m-1}$  \\  
    \hline
    $x^1 = $ &   $00\ldots 0$ & $00\ldots 0$ & $\ldots$ & $00\ldots 0$  \\ 
       \hline
    $x^2 = $ &   $00\ldots 0$ & $00\ldots 0$ & $\ldots$ & $11\ldots 1$  \\ 
    \hline
    $\vdots$ &   $\vdots$ & $\vdots$ & $\vdots$ & $\vdots$ \\
\hline
     $x^{m-1} = $ &   $00\ldots 0$ & $11\ldots 1$ & $\ldots$ & $11\ldots 1$ \\
     \hline
     $x^{m} = $ &   $11\ldots 1$ & $11\ldots 1$ & $\ldots$ & $11\ldots 1$
  \end{tabular}
  \end{center}

  Consider an affine mapping from $\mathbb{R}^{m-1}$ to $\mathbb{R}^n$, given my:
  \[(y_1, \ldots, y_{m-1}) \mapsto (\underbrace{y_1,y_1,\ldots, y_1}_{B_1}, \underbrace{y_2,y_2,\ldots, y_2}_{B_2}, \ldots, \underbrace{y_{m-1},\ldots, y_{m-1}}_{B_{m-1}}) \]
(positions before $B_1$ are set to 0 and after $B_{m-1}$ to 1). Under this mapping, ordered sequences in $\{0,1\}^{m-1}$ go into $x^1, \ldots, x^m$ and thus inside $H_1\cap \ldots \cap H_C$, and non-ordered sequence in $\{0, 1\}^{m-1}$ go into non-ordered sequences in $\{0, 1\}^{n}$ -- outside $H_1\cap \ldots \cap H_C$. Rewriting inequalities, defining $H_1, \ldots, H_C$, in coordinates $y_1, \ldots, y_{m - 1}$, we obtain that $\theta_{m - 1}\le C$. Since $n$ can be chosen arbitrarily and $m\ge n/2^C$, we obtain the inequality $\theta_{m - 1} \le C$ for infinitely many $m$.  

\end{proof}

\subsection{Construction with 2 Heads}
\label{subsec_2head}

We use the following input embedding:
\[E(x_i, i, n) = \begin{pmatrix}
    i\\
    i^2\\
    x_i
\end{pmatrix}, \qquad i = 1, \ldots, n. \]
The two heads compute the following attention logits:
\[L_{in}^{(1)} = in - x_i n^2, \qquad L_{in}^{(2)} = x_i n^2 - in.\]
Let $i_0$ be the position of the last $0$ and $i_1$ be the position of the first $1$. It is easy to see that the attention logits of the first and second head are maximized at positions $i_0$ and $i_1$, respectively, with a margin at least $n$:
\begin{align}
\label{eq_h1}
    L_{i_0 n}^{(1)} &\ge L_{in}^{(1)} + n\qquad i \neq i_0,\\
    \label{eq_h2}
    L_{i_1 n}^{(2)} &\ge L_{in}^{(2)} + n\qquad i \neq i_1. 
\end{align}
The first and the second head compute, respectively,  
\[i_0^* = \frac{\sum\limits_{i = 1}^n \exp\{L_{in}^{(1)} \}i }{\sum\limits_{i = 1}^n \exp\{L_{in}^{(1)}\}}, \qquad  i_1^* = \frac{\sum\limits_{i = 1}^n \exp\{L_{in}^{(2)} \}i }{\sum\limits_{i = 1}^n \exp\{L_{in}^{(2)}\}}.\]
Due to (\ref{eq_h1}--\ref{eq_h2}), one can show that $|i_0 - i_0^*|< 0.1$, $|i_1 - i_1^*|< 0.1$. In fact, one can upper bound these differences by $poly(n) e^{-n}$ but the $0.1$-upper bound is sufficient.  The sequence is ordered if and only if $i_0 + 1 = i_1$. The output MLP checks if $|i_1^* - 1 - i_0^*| < 0.2$.

One has to separately consider cases when the input has only 0s or only 1s, that is, when either $i_0$ or $i_1$ are not defined. It is easy to see that when $x_1 = \ldots = x_n$,  the attention logits of the first head is maximized at $i = n$, and of the second head at $i = 1$, both with a margin at least $n$. That is,  $i_0^*$ will be $0.1$-close to $n$ and $i_1^*$ will be $0.1$-close to $1$. To distinguish it from the when the sequence is of the form$1*\ldots *0$, our attention heads additionally compute:
\[x_n^* = \frac{\sum\limits_{i = 1}^n \exp\{L_{in}^{(1)} \}x_i }{\sum\limits_{i = 1}^n \exp\{L_{in}^{(1)}\}}, \qquad  x_1^* = \frac{\sum\limits_{i = 1}^n \exp\{L_{in}^{(2)} \}x_i }{\sum\limits_{i = 1}^n \exp\{L_{in}^{(2)}\}}.\]

When $L_{in}^{(1)}$ is maximized at $i = n$ and $L_{in}^{(2)}$ is maximized at $i = 1$, the values of $x_1^*, x_n^*$ will be $0.1$-close to $x_1, x_n$ respectively. Hence, the output MLP can take into account the case $x_1 = \ldots =x_n$ as follows. If $|i_0^* - 1| + |i_1^* - n| < 0.2$, it compares $x_1^*, x_n^*$. If they are very close, it outputs that the sequences are ordered, if not -- that they are not. Otherwise,  $|i_0^* - 1| + |i_1^* - n| \ge 0.2$, it outputs that the sequences are ordered if and only if $|i_1^* - 1 - i_0^*| < 0.2$.

\subsection{Proof of Lemma \ref{lem_geom}}
\label{subsec_geom}
Let $G_n$ be a graph whose nodes are non-ordered $n$-dimensional binary vectors, and two such vectors $x,y$ if the segment between them intersects the convex hull of the set of ordered binary vectors.

\begin{claim} For any $n\in\mathbb{N}$, 
        the graph $G_n$ is $\theta_n$-colorable.
    \end{claim}
    \begin{proof}
    Denote $k = \theta_n$, and let $H_1, \ldots, H_k\subseteq\mathbb{R}^n$ by $k$ halfspaces such that all ordered $n$-dimensional binary vectors belong to $H_1\cap H_2 \cap \ldots \cap H_k$ but none of the non-ordered $n$-dimensional non-ordered binary vectors (i.e., none of the vertices of $G_n$) does not belong to this intersection.  Hence, for every vertex $x$ of $G_n$ there is $i\in\{1, 2,\ldots, k\}$ such that $x\in\mathbb{R}^n\setminus H_i$. Let this index $i$ be the color of $x$. If there are more than one index with such property, choose the smallest one.

    We have to show that if two vertices $x,y$ of $G_n$ are connected, then they have a different color. That is, we have to show that they cannot both belong to $\mathbb{R}^n\setminus H_i$ for the same $i\in\{1, \ldots, k\}$.  Indeed, if $x,y$ both belong to $\mathbb{R}^n\setminus H_i$, then so does the segment between them. Hence, this segment cannot intersect the convex hull of the set of ordered binary vectors  as this convex hull is a subset of $H_1\cap \ldots \cap H_k$. 
    \end{proof}

    We now give the lower bound on the chromatic number of $G_n$. We start of by observing that $G_n$ has the following subgraph $G_n'$. Vertices of $G_n'$ are the same as in $G_n$, where two vertices $x,y$ are connected if and only if $x + y$ is an ordered vector. \emph{(For example, binary vectors:
    \[x = (0,0,1,0,1), \qquad y = (0, 1,0,1,1),\]
    are non-ordered, but their sum:
    \[x + y = (0, 1, 1,1,2)\]
    is an ordered vector)}. We have to show that $G_n'$ is a subgraph of $G_n$, that is, if for two vertices $x,y$ of $G_n$ we have that $x + y$ is ordered, then the segment between $x$ and $y$ intersects the convex hull of the set of ordered binary vectors. Indeed, if $x+ y$ is ordered, it first has a block of $0$s, then a block of 1s, and then a block of 2s (with some of these blocks being potentially empty). Hence, we can write:
\[\frac{x + y}{2} = \left(\underbrace{0,0,\ldots, 0}_a, \underbrace{\frac{1}{2},\frac{1}{2},\ldots, \frac{1}{2}}_b, \underbrace{1,1,\ldots, 1}_c\right)\]
for some $a,b,c\ge 0$. Therefore,
\begin{align*}
    \frac{x + y}{2} &= \frac{1}{2}\cdot \left(\underbrace{0,0,\ldots, 0}_a, \underbrace{0,0,\ldots, 0}_b, \underbrace{1,1,\ldots, 1}_c\right)\\ &+   \frac{1}{2}\cdot \left(\underbrace{0,0,\ldots, 0}_a, \underbrace{1,1,\ldots, 1}_b, \underbrace{1,1,\ldots, 1}_c\right),
\end{align*}
meaning that the middle point of the segment between $x$ and $y$ intersects the convex hull of the set of binary ordered vectors. 

Further, consider a restriction of $G_n'$ to  vertices $x$ that consist of precisely 4 blocks of equal bits, starting from the block of $0$s, that is, those that have a form:
    \[x = \left(\underbrace{0,0,\ldots, 0}_a, \underbrace{1,1,\ldots, 1}_b, \underbrace{0,0,\ldots, 0}_c,\underbrace{1,1,\ldots, 1}_d\right)\]
    for some $a,b,c,d > 0$. Each such vertex is uniquely defined by a 3-element subset $\{a, a + b, a + b +c\}\subseteq\{1,2,\ldots, n - 1\}$.

Now, take any $i,j,k,\ell \in\{1,2,\ldots, n - 1\}$  such that:
  \[i < j < k < \ell.\]
  We claim that vertices $x,y$ that correspond to subsets $\{i,j,k\}$ and $\{j,k,\ell\}$, respectively, are connected in $G_n'$. Indeed, dividing positions into blocks of lengths
  \[i, \qquad j - i, \qquad k - j, \qquad \ell - k, \qquad n - \ell,\]
we can write:
\begin{center}
  \begin{tabular}{c|c|c|c|c|c}
       &  $i$ & $j - i$ & $k - j$ & $\ell - k$ & $n - \ell$\\
       \hline
       $x=$ &$0\ldots 0$   &$1\ldots 1$   &$0\ldots 0$   &$1\ldots 1$   &$1\ldots 1$  \\
             \hline
       $y=$ &$0\ldots 0$   &$0\ldots 0$   &$1\ldots 1$   &$0\ldots 0$   &$1\ldots 1$  
  \end{tabular}
\end{center}
One can see that the sum of $x,y$  is ordered.

We conclude that the graph $G_n'$ has a subgraph, isomorphic to the graph $T_n$, defined as follows. Vertices of $T_n$ are 3-element subsets of $\{1,2,\ldots, n\}$. Two subsets are connected if the two largest element in one coincide with two smallest element in the other. It is now enough to show that the chromatic number of $T_n$ goes to $+\infty$ as $n\to+\infty$. That is, for every $q\in\mathbb{N}$, we have to show that there exists $n_0\in\mathbb{N}$ such that $T_n$ is not $q$-colorable for every $n\ge n_0$.

This is an easy consequence of the Ramsey theorem for hypergraphs:

\begin{theorem}[\cite{erdos1952combinatorial}] Let $s > t\ge 2$ and $q$ be natural numbers. Then there exists $n_0$ such that for every $n\ge n_0$ the following holds. For every coloring of the $t$-element subsets of $\{1, 2, \ldots, n\}$ in $q$ colors there exists an $s$-element subset of $\{1,2,\ldots, n\}$ such that all its $t$-element subsets have the same color.
    
\end{theorem}

Fix any $q$ and take the bound for $n_0$ from the last theorem for $s = 4, t = 3$. Consider any $n\ge n_0$ and assume for contradiction that $T_n$ is $q$-colorable. Apply the result to the corresponding coloring of the vertices of $T_n$. There will be  a 4-element subset $\{i, j, k,\ell\}$ such that all its 3-element subsets are colored in the same way. However, if $i < i < k < \ell$, then $\{i, j,k\}, \{j,k,\ell\}$ are connected in $T_n$ and hence cannot be colored by the same color, a contradiction.

\begin{remark}
    The idea to use Ramsey theorem for hypergraphs for the lower bound on the chromatic number of $T_n$ was obtained in a conversation with an LLM. All the other parts of the proof and the writing of the paper have been performed without the AI assistance.
\end{remark}

\end{document}